\documentclass{article}
\usepackage[preprint]{neurips_2025}

\usepackage[utf8]{inputenc} %
\usepackage[T1]{fontenc}    %
\usepackage{hyperref}       %
\usepackage{url}            %
\usepackage{booktabs}       %
\usepackage{amsfonts}       %
\usepackage{nicefrac}       %
\usepackage{microtype}      %
\usepackage{xcolor}         %

\usepackage{graphicx} %
\usepackage{amsmath}
\usepackage{amssymb}
\usepackage{amsthm}
\usepackage{cleveref}
\usepackage{subcaption}

\newtheorem{theorem}{Theorem}
\newtheorem{lemma}{Lemma}

\title{Beyond Diagonal Covariance: Flexible Posterior VAEs via Free-Form Injective Flows}

\author{%
  Peter Sorrenson \\
  Computer Vision and Learning Lab\\
  Heidelberg University, Germany\\
  \texttt{peter.sorrenson@gmail.com} \\
  \And
  Lukas Lührs \\
  Computer Vision and Learning Lab \\
  Heidelberg University, Germany \\
  \AND
  Hans Olischläger \\
  Computer Vision and Learning Lab \\
  Heidelberg University, Germany \\
  \And
  Ullrich Köthe \\
  Computer Vision and Learning Lab \\
  Heidelberg University, Germany \\
}

\begin{document}

\maketitle

\begin{abstract}
    Variational Autoencoders (VAEs) are powerful generative models widely used for learning interpretable latent spaces, quantifying uncertainty, and compressing data for downstream generative tasks. VAEs typically rely on diagonal Gaussian posteriors due to computational constraints. Using arguments grounded in differential geometry, we demonstrate inherent limitations in the representational capacity of diagonal covariance VAEs, as illustrated by explicit low-dimensional examples.
    In response, we show that a regularized variant of the recently introduced Free-form Injective Flow (FIF) can be interpreted as a VAE featuring a highly flexible, implicitly defined posterior. Crucially, this regularization yields a posterior equivalent to a full Gaussian covariance distribution, yet maintains computational costs comparable to standard diagonal covariance VAEs.
    Experiments on image datasets validate our approach, demonstrating that incorporating full covariance substantially improves model likelihood.
\end{abstract}

\section{Introduction}
\label{sec:intro}

Deep generative models have reshaped the way we represent, compress, and create high–dimensional data such as images, audio, or text.  
Among them, \emph{Variational Autoencoders} (VAEs) \citep{kingma2014auto,rezende2014stochastic} stand out for their elegant combination of amortized inference and latent-variable modeling, enabling smooth, semantically meaningful latent spaces and principled uncertainty estimates.  
Yet, in practical VAE implementations the approximate posterior is almost invariably chosen to be a diagonal Gaussian.  
This design choice is driven by the simplicity of the reparameterization trick and a closed-form Kullback-Leibler (KL) term, but it severely limits the posterior’s capacity to capture correlations across latent dimensions and thus constrains the kinds of data manifolds the model can learn.  

\paragraph{Beyond diagonal Gaussians.}
Several research lines have tried to lift the diagonal constraint.  
Hierarchical VAEs \citep{sonderby2016ladder,maaloe2019biva} and sophisticated priors such as VampPrior \citep{tomczak2018vae} enrich the latent structure enrich the latent structure while still using diagonal Gaussians at the level of each latent block.
Normalizing-flow posteriors \citep{rezende2015variational,kingma2016improved} provide a strictly more expressive family at the price of stacking many flow steps whose Jacobians must be triangular or specially constructed for tractability.  
Full-covariance Gaussians can be learned directly via Cholesky factors \citep{kingma2019introduction} or by Householder transformations \citep{tomczak2016improving}, yet both scale quadratically in latent dimension.
The Variational Laplace Autoencoder (VLAE) \citep{park2019variational} solves an optimization problem to estimate the mean and (potentially non-diagonal) covariance of the posterior, but involves expensive Jacobian evaluation and matrix inversion which does not scale well with latent dimension.

\paragraph{Injective flows as manifold learners.}
A complementary line of work adapts exact likelihood normalizing flows \citep{rezende2015variational,kobyzev2021normalizing,papamakarios2021normalizing} to learn a manifold embedded in the data space, trading a bijectivity constraint for injectivity.  
Such injective normalizing flows simultaneously learn and maximize likelihood on a manifold by jointly optimizing an encoder-decoder pair together with the change-of-variables term \citep{brehmer2020flows,caterini2021rectangular}.

Very recently, \emph{Free-form Injective Flows} (FIFs) \citep{sorrenson2024lifting,draxler2024freeform} removed the architectural constraints of injective normalizing flows by introducing an efficient Hutchinson-style estimator for the gradient of the log-determinant.  
FIFs learn both manifold and density simultaneously, but they provide no direct latent-variable interpretation and no tractable posterior.

\paragraph{This paper.}
We observe that, under a mild Laplace approximation, the geometry of FIFs naturally induces a \emph{full-covariance} Gaussian posterior whose covariance is given implicitly through encoder and decoder Jacobians.  
Building on this insight, we introduce FIVE, a \textbf{F}ree-form \textbf{I}njective-Flow-based \textbf{V}ariational Auto\textbf{E}ncoder that:

\begin{enumerate}
    \item Leverages FIFs to furnish VAEs with a highly flexible posterior that is \emph{computationally only slightly more expensive} than the diagonal case.
    \item Provides a differential-geometric analysis revealing when and why diagonal posteriors fail to represent curved data manifolds.
    \item Proposes a curvature-aware regularization that stabilizes training and yields an approximate KL term without computing determinants explicitly.
    \item Matches or surpasses the log-likelihood of both VAEs and standalone FIFs on benchmark image datasets.
\end{enumerate}

Taken together, our results indicate that moving \emph{beyond} diagonal covariance is not only theoretically well-motivated but also practically attainable at scale.  We hope that FIVE will serve as a drop-in replacement for standard VAEs when richer posteriors estimates are desired.

\section{Variational Autoencoders and their limitations}
\label{sec:vae}

Variational Autoencoders (VAEs) are latent-variable models characterized by a prior distribution over latent variables $p(z)$ (typically standard normal) and a conditionally normal distribution over data points given latent variables of the form:
\begin{equation}
    p(x|z) = \mathcal{N}(x; g(z), \sigma^2 I),
\end{equation}
where $g(z)$ is a neural network decoder, and $\sigma^2$ is a scalar controlling the variance of the likelihood.

The marginal distribution $p(x)$ is obtained by integrating over latent variables:
\begin{equation}
    p(x) = \int p(x|z)p(z)\, dz.
\end{equation}
This integral is typically intractable, making direct optimization of the KL divergence $D_\text{KL}(q(x)\|p(x))$ between the training distribution $q(x)$ and the model distribution $p(x)$ difficult.

\subsection{Introducing the variational distribution}

To circumvent the computational difficulty, VAEs introduce an auxiliary distribution, called the variational posterior, $q(z|x)$. By incorporating this distribution, we have two distributions defined over the joint latent-data space $(x, z)$:
\begin{align}
    p(x,z) &= p(z)p(x|z), \\
    q(x,z) &= q(x)q(z|x).
\end{align}

Instead of minimizing $D_\text{KL}(q(x)\|p(x))$ directly, we minimize the divergence between these two joint distributions:
\begin{equation}
\label{eq:kl-joint-x-z}
    D_\text{KL}(q(x,z)\|p(x,z)) = D_\text{KL}(q(x)\|p(x)) + \mathbb{E}_{q(x)}\left[D_\text{KL}(q(z|x)\|p(z|x))\right].
\end{equation}

This divergence provides an upper bound on $D_\text{KL}(q(x)\|p(x))$ due to the non-negativity of KL divergence. Explicitly, we have:
\begin{equation}
    D_\text{KL}(q(x)\|p(x)) \leq D_\text{KL}(q(x,z)\|p(x,z)),
\end{equation}
with the gap precisely given by:
\begin{equation}
    \mathbb{E}_{q(x)}[D_\text{KL}(q(z|x)\|p(z|x))].
\end{equation}

\subsection{Evidence Lower Bound (ELBO)}

Typically, VAEs assume that the variational distribution $q(z|x)$ is a diagonal Gaussian:
\begin{equation}
    q(z|x) = \mathcal{N}(z; \mu(x), \text{diag}(\sigma^2(x))),
\end{equation}
for computational tractability. If we discard the unknown entropy of $q(x)$, we can rewrite \cref{eq:kl-joint-x-z} as the well-known Evidence Lower Bound (ELBO):
\begin{equation}
    \text{ELBO}(x) = \mathbb{E}_{q(z|x)}[\log p(x|z)] - D_{\text{KL}}(q(z|x)\|p(z)).
\end{equation}
Since typically $p(z)$ is the standard normal, we can evaluate the KL-divergence via the analytical expression:
\begin{equation}
    D_{\text{KL}}(\mathcal{N}(\mu, \Sigma)\|\mathcal{N}(0, I_d)) = \frac{1}{2} \left( \|\mu\|^2 + \operatorname{tr}(\Sigma) - d - \log\det(\Sigma) \right)
\end{equation}
which is easy to evaluate for diagonal $\Sigma$.

Optimization then becomes a maximization of the ELBO, here equivalently written as a minimization of the negative ELBO:
\begin{equation}
    \mathcal{L}(x) = \mathbb{E}_{q(z|x)}[-\log p(x|z)] + D_{\text{KL}}(q(z|x)\|p(z)).
\end{equation}

\subsection{Form of the posterior and its implications}

Since $q(z|x)$ is assumed diagonal Gaussian, the optimal posterior $p(z|x)$ which minimizes $D_\text{KL}(q(z|x)\|p(z|x))$ will also be approximately diagonal Gaussian. This assumption significantly constrains the decoder function $g(z)$.

Specifically, for a Gaussian posterior approximation, the following theorem applies:

\begin{theorem}
    When optimizing the above described latent variable model (standard VAE) with diagonal Gaussian $q(z|x)$ and injective decoder, the optimal decoder $g$ has an approximately diagonal pull-back metric, i.e., $g'(z)^\top g'(z)$ is approximately diagonal for all $z$, where $g'(z) = \frac{\partial g(z)}{\partial z}$ is the Jacobian of the decoder.
\end{theorem}

The full proof is in the appendix.

\textit{Proof sketch:}
If the loss is minimal, then due to the decomposition in \cref{eq:kl-joint-x-z}, $D_\text{KL}(q(z|x)\|p(z|x)) = 0$. This implies that $p(z|x) = q(z|x)$ and hence $p(z|x)$ is diagonal Gaussian. Strictly speaking, the requirement that $p(z|x)$ is Gaussian implies that $g$ is linear. In practice, $g$ typically needs to be nonlinear due to the structure of the data, but we will find that $g$ needs to ``look linear'' on the scale of $\sigma$. This is considered in more detail in the appendix, where we show that second and higher-order derivatives of $g$ must be small.

To find the form of the Gaussian approximation to $p(z|x)$, we approximate $-\log p(x,z)$ as a quadratic in $z$. This leads us to the following solution, where we only keep first order derivatives of $g$, due to the argument that higher-order derivatives must be small. Let $f(x) = \arg\min_z -\log p(x,z)$. Then
\begin{equation}
    p(z|x) \approx \mathcal{N}\left(z; \mu=f(x), \Sigma=\left(I + \frac{1}{\sigma^2}g'(f(x))^\top g'(f(x))\right)^{-1}\right)
\end{equation}

Since $p(z|x) \approx q(z|x)$ and $q(z|x)$ has diagonal covariance, this implies that 
\begin{equation}
    \Sigma=\left(I + \frac{1}{\sigma^2}g'(f(x))^\top g'(f(x))\right)^{-1}    
\end{equation}
is approximately diagonal, which can only be the case if $g'(f(x))^\top g'(f(x))$ is approximately diagonal. The decoder is injective (by assumption), leading to surjective $f$. Hence, the statement that $g'(z)^\top g'(z)$ is approximately diagonal holds for all $z$.

This requirement implies severe restrictions on the decoder, which are explored in the next section.

\subsection{Existence of orthogonal coordinates}

The above results tell us that standard VAEs can only represent functions $g$ such that $g'(z)^\top g'(z)$ is approximately diagonal. Obviously, we can construct functions such that this is not the case, but perhaps it is possible to reparameterize the function such that the Jacobian in the new coordinates has the desired property. To do this, define new coordinates $u = \varphi(z)$ and a new function $h = g \circ \varphi^{-1}$. Clearly $h(u) = g(z)$ and therefore $h$ and $g$ have the same image. In other words, $h$ spans the same manifold in data space as $g$. In addition, $h$ has a different Jacobian: $h'(u) = g'(z) \varphi^{-1\prime}(u)$ and we can search for $h$ such that $h'(u)^\top h'(u)$ is diagonal.

In differential geometry, objects of the form $g'(z)^\top g'(z)$ are known to be metric tensors, in particular, $G = g'(z)^\top g'(z)$ is called the pull-back metric of $g$, which transfers the geometry of the manifold spanned by $g$ into the $z$-space. Our problem can then be framed in the language of differential geometry as the search for an orthogonal coordinate system, given metric tensor $G$. It is well known that it is always possible to find orthogonal (and indeed conformal) local coordinates for any 2-dimensional metric (existence of isothermal coordinates). However, in 3 dimensions and higher, there can be obstructions to forming local orthogonal coordinate systems. 

It is informative to first consider how we form such systems in 2 dimensions. In this case, finding a coordinate system is equivalent to partitioning the space into level sets of two curves $u^1(z)$ and $u^2(z)$. The basis vectors in the tangent space are the gradients, e.g., $e_1 = \nabla_z u^1(z)$. For the metric to be diagonal in these coordinates, we require that cross terms like $G(e_1, e_2)$ vanish, meaning that $e_1$ and $e_2$ should be the eigenvectors of $G$. This in turn means that the level sets, e.g., $u^1(z) = 0$, need to be perpendicular to the eigenvectors of $G$. In 2 dimensions this is always possible locally: we are finding a curve perpendicular to a smoothly varying vector field (e.g., $e_1$), which is satisfied by the flow lines of the perpendicular field given by the other eigenvector. See \cref{fig:2d_coord_systems} for an illustration.

\newlength{\figheight}
\setlength{\figheight}{6cm}  

\begin{figure}[t]
  \centering
  \begin{subfigure}[b]{0.3\textwidth}
    \centering
    \includegraphics[height=\figheight,keepaspectratio]{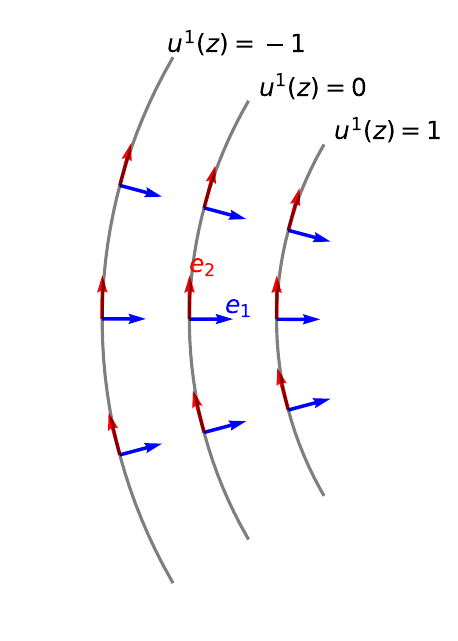}
    \caption{Level sets of $u^1$, perpendicular to $e_1$, found by integrating $e_2$.}
  \end{subfigure}%
  \hfill
  \begin{subfigure}[b]{0.3\textwidth}
    \centering
    \includegraphics[height=\figheight,keepaspectratio]{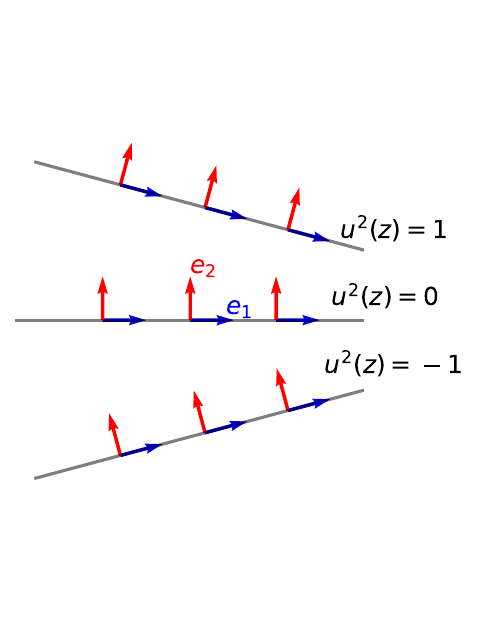}
    \caption{Level sets of $u^2$, perpendicular to $e_2$, found by integrating $e_1$.}
  \end{subfigure}%
  \hfill
  \begin{subfigure}[b]{0.3\textwidth}
    \centering
    \includegraphics[height=\figheight,keepaspectratio]{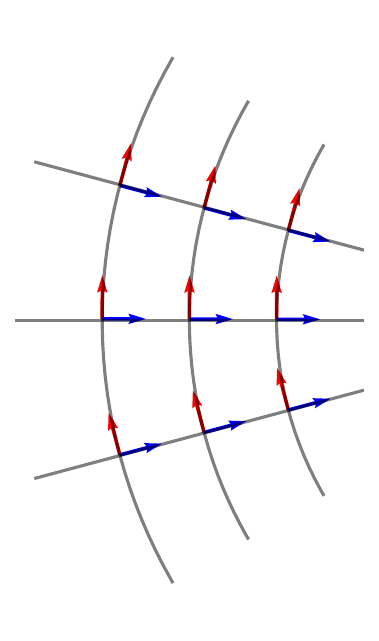}
    \caption{The resulting orthogonal coordinate system.}
  \end{subfigure}
  \caption{Construction of a local orthogonal coordinate system in 2 dimensions from eigenvectors $e_1$ and $e_2$ of the metric tensor $G$.}
  \label{fig:2d_coord_systems}
\end{figure}

In 3 dimensions the situation is more complicated. We need to follow essentially the same procedure: i.e., form surfaces perpendicular to the eigenvectors of $G$. If this is possible, it will be satisfied by the integral curves of the other two eigenvectors, i.e., the surface formed by stitching together the infinitesimal planes spanned by these vectors. But we can only do this if these planes all line up: if the planes twist around each other in an uncoordinated way, there is no consistent way to form the surface without tearing. The condition that these planes line up and hence that they can be integrated into a surface, is given by the Frobenius integrability theorem \citep[Ch. 19]{lee2013manifolds} and requires a property known as involutivity. This says that if you first integrate a short distance along $e_1$ and then along $e_2$, you should end up in the same place as integrating in the reverse order, up to a difference that lies in the plane spanned by $e_1$ and $e_2$ (see \cref{fig:involutivity}). Only if every pair of eigenvectors is involutive, can we form a coordinate system from them. Since using the eigenvectors of $G$ is the only way to form an orthogonal coordinate system, without involutivity, it is not possible to form orthogonal coordinates.

\begin{figure}[t]
  \centering
  \begin{subfigure}[b]{0.5\textwidth}
    \centering
    \includegraphics[width=\linewidth, trim=0cm 3cm 0cm 0cm, clip]{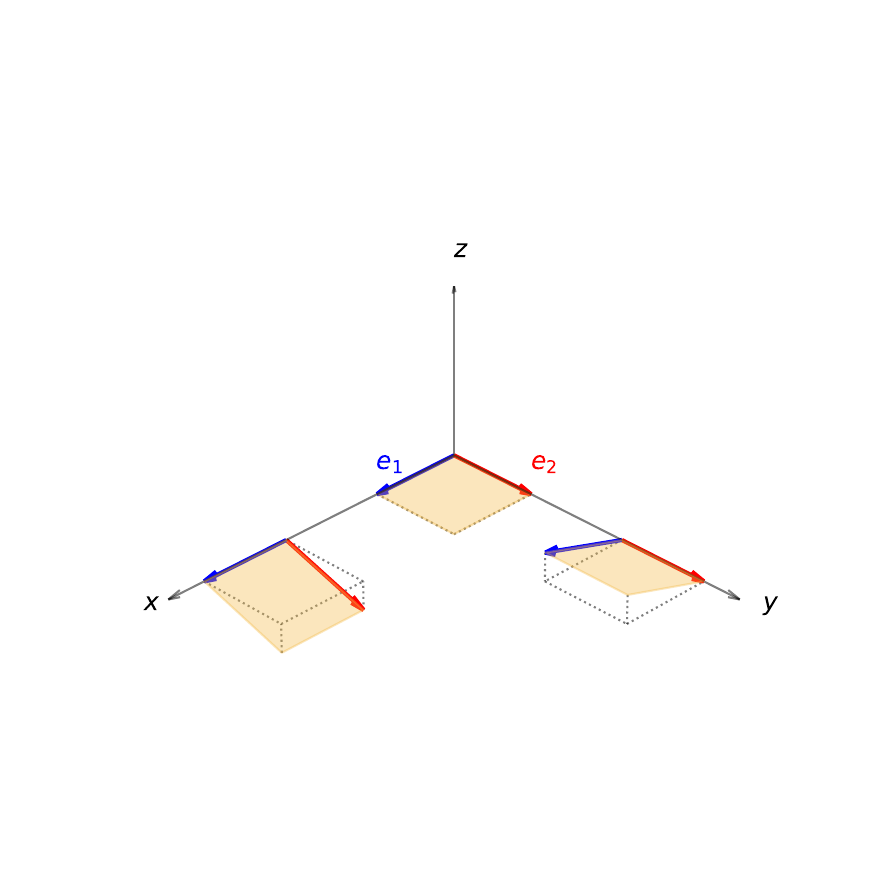}
  \end{subfigure}%
  \hfill
  \begin{subfigure}[b]{0.5\textwidth}
    \centering
    \includegraphics[width=\linewidth, trim=0cm 3cm 0cm 0cm, clip]{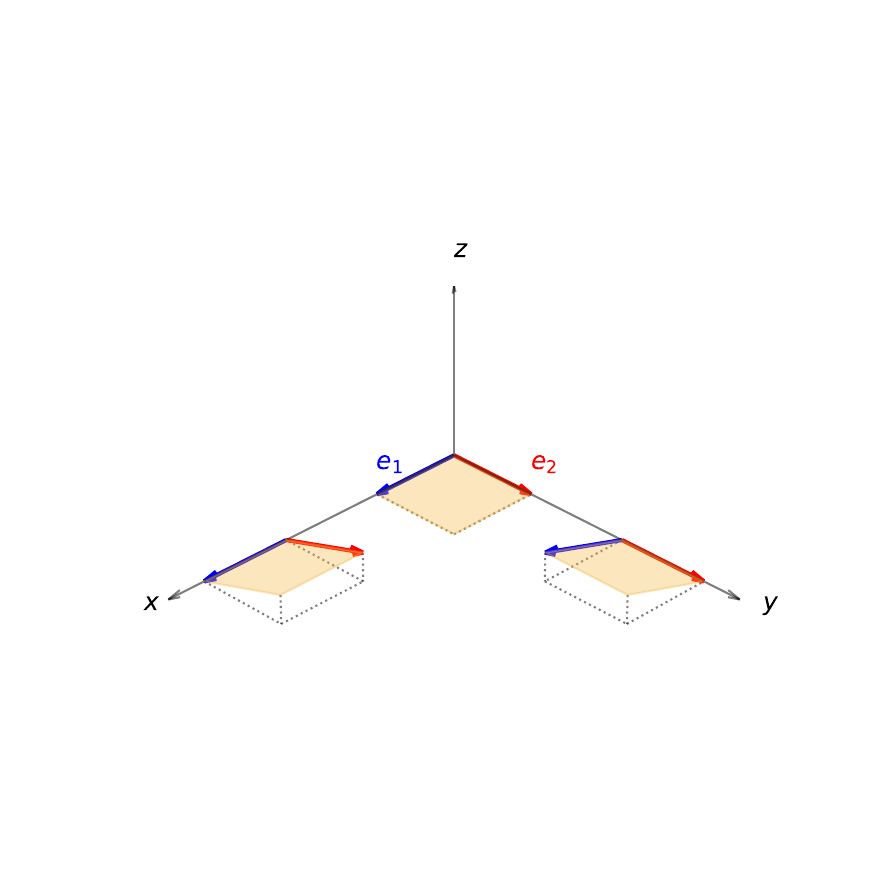}
  \end{subfigure}%
  \caption{Left: Illustration of non-involutive eigen-lines. The planes spanned by $e_1$ and $e_2$ cannot be integrated together into a consistent surface (the surface formed by extending the right-hand plane in the $e_1$ direction lies above that formed by extending the left-hand plane in the $e_2$ direction). Right: The planes can be integrated to form the surface $z=xy$, hence they are involutive.}
  \label{fig:involutivity}
\end{figure}

This is stated formally in the following theorem.

\begin{theorem}[Local existence of orthogonal coordinates]
Let
\begin{equation}
  g : U \subset \mathcal{Z} \to \mathcal{X}
\end{equation}
be a smooth immersion on a connected open set $U$ (so $\operatorname{rank} g'(z) = d$ everywhere).
Set
\begin{equation}
  G = J_{g}^{\top} J_{g},
\end{equation}
the induced positive–definite metric on $U$.
Assume that every eigenvalue of $G$ is simple and varies smoothly, so one may choose a smooth orthonormal eigen-frame $\{e_{1},\dots,e_{d}\}$ with $G(e_{i},e_{i}) = \lambda_{i} > 0$.
For each $i$ define the rank-$1$ eigen-line field $\mathcal{L}_{i} = \operatorname{span}\{e_{i}\} \subset TU$.

The following statements are equivalent:
\begin{enumerate}
  \item \textbf{Orthogonal coordinates exist.}\;
        For every point $p \in U$ there is a neighborhood $V$ around $p$ and a diffeomorphism
        $\varphi : V \to W \subset \mathcal{X}$,
        $\;u = (u^{1},\dots,u^{d}) = \varphi(z)$, such that for
        $h := g \circ \varphi^{-1}$ the pull-back metric is diagonal:
        \begin{equation}
          h^{*}\langle\cdot,\cdot\rangle
          \;=\;
          \sum_{i=1}^{d} H_{i}^{2}(u)\,(du^{i})^{2},
          \qquad H_{i} > 0 \text{ smooth}.
        \end{equation}

  \item \textbf{Eigen-line fields are involutive.}\;
        For all $i \neq j$ one has
        \begin{equation}
          [e_{i},e_{j}] \;\in\; \operatorname{span}\{e_{i},e_{j}\}.
        \end{equation}
\end{enumerate}
\end{theorem}

The proof is given in the appendix.

As a result, if the eigen-line fields are not involutive at any point $p$, it is not possible to form orthogonal coordinates locally (and hence globally), and there is no function $h$ such that both $h(\mathcal{Z}) = g(\mathcal{Z})$ and the pull-back metric $J_h^{\top} J_h$ is diagonal. Hence, the manifold spanned by $g$ cannot be well represented by a VAE with a diagonal Gaussian posterior.

\section{Free-form Injective Flows}
\label{sec:fif}

Free-form Injective Flows are generative models that learn a manifold and maximize likelihood on it. In this section, we provide a new derivation of the training objective using a Laplace approximation.

Assume the same generative model as in VAEs, with a standard normal prior $p(z)$ and a Gaussian decoder $p(x|z)$. Define the joint negative log-likelihood as
\begin{equation}
u(x, z) = -\log p(x, z) = -\log p(x|z) - \log p(z).
\end{equation}

Let
\begin{equation}
f(x) = \arg\min_z u(x, z),
\end{equation}
so $f(x)$ is the mode of the posterior $p(z|x)$. If we assume that the second and higher-order derivatives are small, we can apply a Laplace approximation:
\begin{equation}
u(x, z) \approx u(x, f(x)) + \frac{1}{2} (z - f(x))^\top H(x) (z - f(x)),
\end{equation}
where $H(x)$ is the Hessian of $u(x, z)$ with respect to $z$, evaluated at $z = f(x)$.

This implies the marginal likelihood is approximately
\begin{align}
p(x) &= \int \exp(-u(x, z)) \, dz \\
&\approx \exp(-u(x, f(x))) \int \exp\left(-\frac{1}{2} (z - f(x))^\top H(x) (z - f(x))\right) dz \\
&= \exp(-u(x, f(x))) \cdot \det(H(x)^{-1})^{1/2}.
\end{align}

Taking the negative log, we get an approximate expression for the negative log likelihood:
\begin{equation}
-\log p(x) \approx u(x, f(x)) + \frac{1}{2} \log \det H(x).
\end{equation}

\subsection{Gradient of the determinant term}

We would like to use the estimate of $-\log p(x)$ as a training objective for the decoder. Of the two terms, the determinant term is difficult to optimize directly. Using Jacobi’s formula, the gradient of the log-determinant with respect to parameters of $g$ is
\begin{equation}
\nabla \log \det H(x) = \operatorname{tr}(H(x)^{-1} \nabla H(x)).
\end{equation}

Under our approximation, we use the relationship:
\begin{equation}
H(x)^{-1} \approx \sigma^2 f'(x) (g'(f(x)))^{+\top},
\end{equation}
where $+$ denotes the pseudoinverse. Then the gradient becomes approximately
\begin{equation}
\nabla \frac{1}{2} \log \det H(x) \approx \nabla \operatorname{tr} \left( \text{SG}\left[f'(x)\right] \cdot g'(f(x)) \right),
\end{equation}
where SG denotes stop-gradient. See the appendix for a detailed derivation. Thus, the overall training loss becomes:
\begin{equation}
\mathcal{L}_{\text{FIF}}(x) = u(x, f(x)) + \operatorname{tr}\left(\text{SG}[f'(x)] \cdot g'(f(x))\right).
\end{equation}
As in \citet{sorrenson2024lifting}, the trace term can be efficiently estimated using vector-Jacobian products and Hutchinson-style trace estimation \citep{hutchinson1989stochastic}:
\begin{equation}
    \operatorname{tr}\left(\text{SG}[f'(x)] \cdot g'(f(x))\right) \approx \text{SG}[v^\top f'(x)] g'(f(x)) v
\end{equation}
where $v$ is sampled from a distribution such that $\mathbb{E}[vv^\top] = I$.

\subsection{The need for regularization}

Our derivation assumes that second and higher-order derivatives of $g$ are small. To ensure this holds, we must regularize $g$ appropriately.

Drawing inspiration from VAEs, we can use KL divergence to a Gaussian posterior as a regularizer. As discussed, this naturally supresses second and higher-order derivatives of $g$. The current objective minimizes an approximation of the KL-divergence $D_\text{KL}(q(x)\|p(x))$, equivalent to minimizing $-\log p(x)$ over the data. From \cref{eq:kl-joint-x-z} we see that adding regularization of the form $\mathbb{E}_{q(x)}[D_\text{KL}(q(z|x)\|p(z|x))]$ makes the loss into the standard VAE target $D_\text{KL}(q(x,z)\|p(x,z))$. This provides a practical method to enforce smoothness and control curvature. In the next section, we define a new VAE model that incorporates this idea.

\section{Free-form Injective Flow-based Variational Autoencoders}
\label{sec:five}

In this section we derive our new model FIVE as a synthesis between free-form injective flows and variational autoencoders.

From the Laplace approximation in the previous section we had:
\begin{equation}
p(z|x) \approx \mathcal{N}(z; f(x), H(x)^{-1})
\end{equation}
and
\begin{equation}
    H(x)^{-1} \approx \sigma^2 f'(x) (g'(f(x)))^{+\top} \approx \sigma^2 f'(x) f'(x)^\top.
\end{equation}
Based on this, we define a variational posterior as:
\begin{equation}
q(z|x) = \mathcal{N}(z; f(x), \sigma^2 f'(x) f'(x)^\top).
\end{equation}

\subsection{Sampling}

To sample from $q(z|x)$:
\begin{equation}
z = f(x) + \sigma f'(x) v, \quad \text{with } v \sim \mathcal{N}(0, I).
\end{equation}

\subsection{Loss function: ELBO}

The negative evidence lower bound (ELBO) loss is given by:
\begin{equation}
\mathcal{L}_{\text{ELBO}}(x) = \mathbb{E}_{q(z|x)}[-\log p(x|z)] + D_{\text{KL}}(q(z|x) \| p(z)).
\end{equation}

The KL divergence term evaluates to:
\begin{align}
D_{\text{KL}}(q(z|x) \| p(z)) &= \frac{1}{2} \left( \|f(x)\|^2 + \operatorname{tr}(\sigma^2 f'(x) f'(x)^\top) - d \right. \\
&\left. \quad - \log \det(\sigma^2 f'(x) f'(x)^\top) \right),
\end{align}
where $d$ is the dimensionality of $z$.

The gradient of the log-determinant with respect to encoder parameters is:
\begin{equation}
\nabla \frac{1}{2} \log \det(f'(x) f'(x)^\top) = \operatorname{tr} \left( f'(x)^+ \nabla_x f'(x) \right) \approx \operatorname{tr}(g'(f(x)) \nabla f'(x)).
\end{equation}

\subsection{Consistency of approximations}

In our derivation of the FIVE model, we made two opposing approximations involving the Jacobians of the encoder and decoder:

\begin{enumerate}
    \item When approximating the posterior covariance $H(x)^{-1}$, we used
    \begin{equation}
    H(x)^{-1} \approx \sigma^2 f'(x) f'(x)^\top,
    \end{equation}
    effectively substituting $f'(x)$ for $(g'(f(x)))^+$.
    
    \item In the gradient of the log-determinant term, we used the approximation
    \begin{equation}
    \nabla_x \log \det(f'(x) f'(x)^\top) \approx \operatorname{tr}(g'(f(x)) \nabla_x f'(x)),
    \end{equation}
    which substitutes $g'(f(x))$ for $(f'(x))^+$.
\end{enumerate}

These approximations may appear contradictory at first glance. However, this ``in-out'' substitution turns out to be well-motivated: in fact, we can show that under a linear setting, the resulting model exactly reproduces the training distribution.

\begin{theorem}
Let $f$ and $g$ be linear maps, and let $q(x)$ be a zero-mean distribution with covariance matrix $\Sigma_x$. Then, the global optimum of the FIVE loss corresponds to a model $p(x)$ that is a Gaussian distribution with mean zero and covariance $\Sigma_x$:
\begin{equation}
p(x) = \mathcal{N}(x; 0, \Sigma_x),
\end{equation}
which is the best possible solution attainable by any linear decoder $g$.
\end{theorem}

\noindent This result shows that our opposing approximations are consistent in the linear case, and FIVE recovers the data distribution exactly under optimal conditions. This supports the idea that the approximations, while heuristic in the general nonlinear case, form a sound basis for learning in the regime where $f$ and $g$ behave approximately linearly in local regions of the data manifold.

\section{Experiments}
\label{sec:experiments}

\subsection{Implementation details}

\paragraph{Learnable noise level}

In typical VAE models, $\sigma^2$, the variance of noise added in the data space, is a fixed hyperparameter. We follow the approach of $\sigma$-VAE \citep{rybkin2021simple} and make it learnable in all four models we consider. In practice, we add an additional learnable parameter $\log \sigma$ to our network and make sure we don't drop any terms containing $\sigma$ in the relevant densities (e.g., $\log p(x|z)$).

\paragraph{Importance sampling}

Since the integral $p(x) = \int p(x,z) dz$ is typically intractable, most VAE implementations estimate $\mathbb{E}_{q(x)}[\log p(x)]$ through importance sampling. This uses the identity
\begin{equation}
    p(x) = \int p(x,z) dz = \int \frac{q(z|x)}{q(z|x)} p(x,z) dz = \mathbb{E}_{q(z|x)}\left[\frac{p(x,z)}{q(z|x)}\right]
\end{equation}
which in practice is estimated by $K$ importance samples $z_k \sim q(z|x)$:
\begin{equation}
    p(x) \approx \frac{1}{K} \sum_{k=1}^K \frac{p(x,z_k)}{q(z_k|x)}.
\end{equation}
Due to Jensen's inequality, the estimated mean log density is a lower bound on the true value, with the bound becoming tight as $K \to \infty$.

Since the free-form injective flow (FIF) does not come with a pre-defined variational posterior, we construct $q(z|x)$ in the same way as in FIVE, meaning we define
\begin{equation}
    q(z|x) = \mathcal{N}(z; f(x), \sigma^2 f'(x) f'(x)^\top).
\end{equation}
We find that this is a sufficiently good approximation to estimate likelihoods via importance sampling.

\paragraph{Full-covariance VAE}

Rather than predicting a lower-triangular Cholesky factor as in \cite{kingma2019introduction}, we predict a lower-triangular $A$ which we symmetrize to form $\log \Sigma = (A + A^\top)/2$. We then compute $\Sigma$ via the matrix exponential. This is a convenient parameterization since $\Sigma$ is guaranteed to be positive definite and we can leverage the identity
\begin{equation}
    \log \det \exp(M) = \operatorname{tr}(M)
\end{equation}
for a square matrix $M$.

\subsection{Image data}

We use the same experimental setting as Variational Laplace Autoencoders \citep{park2019variational}, training two network sizes on flattened MNIST and CIFAR10. The small network has a single hidden layer of dimension 256 and latent dimension of 16 for MNIST and double these numbers for CIFAR10. The larger network has two hidden layers of dimension 500, and a latent dimension of 50 (same for both datasets). We use \texttt{ReLU} activations. We train for 2000 epochs using Adam with learning rate $5 \times 10^{-4}$ and batch size 128 and select the best epoch based on a small validation set (size 1000) split from the training set. We report log likelihoods on the held-out test set (see \cref{tab:image-data}).

We train on four models: i) a standard VAE with diagonal covariance posterior \citep{kingma2014auto} ii) a VAE with a full covariance posterior (FC-VAE) \citep{kingma2019introduction} iii) free-form injective flow (FIF) \citep{sorrenson2024lifting} iv) free-form injective flow-based VAE (FIVE). We state our results in \cref{tab:image-data}. We can see that FIVE outperforms or matches the other models, evaluated on test likelihood, in all four experimental settings. 

\begin{table}[t]
\caption{Comparison of different models on MNIST and CIFAR10 using small and larger networks. We report mean log likelihoods on the unseen test set, estimated using importance sampling with 100 samples per data point. Unbiased standard deviations are reported over 3 runs. Best model in bold. Models which are statistically indistinguishable from best model in italics.}
\centering
\label{tab:image-data}
\begin{tabular}{lcccc}
\toprule
 & \multicolumn{2}{c}{\textbf{MNIST}} & \multicolumn{2}{c}{\textbf{CIFAR10}} \\
\cmidrule(lr){2-3} \cmidrule(lr){4-5}
\textbf{Model} & Small network & Larger network & Small network & Larger network \\
\midrule
VAE                     & $729.83 \pm 3.99$ & $935.73 \pm 8.60$ & $2389.46 \pm 9.86$ & $\it{2680.36 \pm 13.86}$ \\
FC-VAE                  & $732.24 \pm 1.96$ & $770.16 \pm 1.22$ & $2388.02 \pm 8.93$ & $2531.18 \pm 53.94$ \\
\midrule
FIF                     & $742.06 \pm 4.30$ & $\bf{1116.31 \pm 27.29}$ & $2363.68 \pm 10.43$ & $2646.96 \pm 14.29$ \\
\midrule
FIVE (ours)     & $\bf{774.41 \pm 1.87}$ & $\bf{1115.17 \pm 7.67}$ & $\bf{2412.02 \pm 7.67}$ & $\bf{2684.44 \pm 9.33}$ \\
\bottomrule
\end{tabular}
\end{table}

\paragraph{MNIST}

FIVE outperforms VAEs in both the small network and larger network settings, with a more pronounced difference when using larger networks. The dominant effect is that the two VAE models suffer from posterior collapse when trained with a larger latent space, leading to underperforming likelihoods. The FC-VAE is particularly prone to posterior collapse. Why FC-VAE is more prone to posterior collapse than the standard VAE, and why FIVE does not suffer from posterior collapse despite its formulation as a VAE are open questions. FIVE is presumably helped by its similarity in form to FIF, which also does not suffer from posterior collapse.

FIVE also outperforms or matches FIF, demonstrating the usefulness of the additional regularization.

\paragraph{CIFAR10}

FIVE is the best performing model in both the small and larger network settings, although in the larger network setting, the standard VAE matches the performance of FIVE. Again we see FC-VAE underperforming the standard VAE when given a larger latent space, due to its higher susceptibility to posterior collapse.

Similarly to the MNIST experiments, FIVE outperforms FIF, showing a distinct benefit of the improved formulation.

\section{Discussion}
\label{sec:discussion}

\paragraph{Summary of contributions.}
We introduced FIVE, a VAE whose posterior covariance is \emph{implicitly} induced by the Jacobians of an encoder–decoder pair trained as a free-form injective flow.  
This yields a full–covariance Gaussian posterior at essentially the same asymptotic cost as the diagonal case, while retaining the geometric fidelity and likelihood performance of FIFs. Our theoretical analysis clarifies \emph{when} diagonal posteriors are fundamentally insufficient, and our empirical study shows consistent gains on image benchmarks.

\subsection*{Limitations}

\begin{itemize}
    \item \textbf{Local‐linearity assumption.}  
    Our Laplace approximation assumes that higher-order derivatives of the decoder are small in the neighborhood explored by the noise scale $\sigma$.  
    When the data manifold exhibits sharp curvature on that scale the approximation may break down, potentially degrading likelihood estimates and sample quality.

    \item \textbf{Dataset scope.}  
    Experiments were limited to MNIST, CIFAR-10 %
    These benchmarks are standard but modest in scale; future work should test FIVE on higher-resolution natural images, 3-D data, and text to confirm generality.
\end{itemize}

\paragraph{Potential societal impact}
Due to the small scale of our experiments, we do not see direct societal impact of our work. However, if scaled up in future work, it may pose similar opportunities and risks as other generative models, e.g., potential for use in scientific discovery, as well as misuse for malicious purposes.

In summary, FIVE demonstrates that moving beyond diagonal covariance is both theoretically principled and practically feasible. We hope this work encourages the community to revisit posterior design in VAEs and to explore geometrically informed generative modeling at scale.

\section*{Acknowledgments}
This work is supported by Deutsche Forschungsgemeinschaft (DFG, German Research Foundation) under Germany's Excellence Strategy EXC-2181/1 - 390900948 (the Heidelberg STRUCTURES Cluster of Excellence). 

\bibliographystyle{plainnat}
\bibliography{references}

\newpage

\appendix
\section{Proofs and derivations}

\subsection{Proof of theorem 1}

\begin{theorem}\label{thm:diagonal-pullback}
    Let $g: \mathbb{R}^d \to \mathbb{R}^n$ be injective, with $d \le n$.
    Let $p(z) = \mathcal{N}(z; 0, I_d)$ and $p(x|z) = \mathcal{N}(x; g(z), \sigma^2 I_n)$. 
    Let $q(z|x) = \mathcal{N}(z; f_\mu(x), \operatorname{diag}(f_{\sigma^2}(x)))$ with $f_\mu, f_{\sigma^2}: \mathbb{R}^n \to \mathbb{R}^d$ be a Gaussian variational posterior with diagonal covariance.
    Suppose we maximize the ELBO with respect to $g$, $f_\mu$ and $f_{\sigma^2}$:
    \begin{equation}
        \text{ELBO}(x) = \mathbb{E}_{q(z|x)}[\log p(x|z)] - D_{\text{KL}}(q(z|x)\|p(z)).
    \end{equation}
    Then, the optimal decoder $g^*$ has an approximately diagonal pull-back metric, i.e., $g^{*\prime}(z)^\top g^{*\prime}(z)$ is approximately diagonal for all $z$, where $g^{*\prime}(z) = \frac{\partial g^*(z)}{\partial z}$ is the Jacobian of the optimal decoder.
\end{theorem}

\begin{proof}
As discussed in the main text, maximization of the ELBO is equivalent to minimization of the KL divergence between $q(x,z)$ and $p(x,z)$, which can be decomposed as
\begin{equation}
    D_\text{KL}(q(x,z)\|p(x,z)) = D_\text{KL}(q(x)\|p(x)) + \mathbb{E}_{q(x)}[D_\text{KL}(q(z|x)\|p(z|x))].
\end{equation}
Hence, the optimal solution will be such that $p(z|x)$ is approximately equal to $q(z|x)$, in order to minimize the second term. In practice this means that second and higher order derivatives of $g$ are suppressed, such that a Laplace (i.e., Gaussian) approximation for $p(z|x)$ is highly accurate.

To that end, define $f(x) = \arg\min_z \tfrac12 \|z\|^2+\frac{1}{2\sigma^2}\|x-g(z)\|^2$ and let $H(x)$ be the Hessian of $-\log p(x,z)$ in $z$, evaluated at $z=f(x)$. Discarding higher-order terms, this leads to
\begin{equation}
    H(x) = I_d+\frac{1}{\sigma^{2}}J^{\top}J,\qquad J=g'(f(x)).
\end{equation}
Hence, for small $\sigma$,
\begin{equation}
    p(z|x)=\mathcal{N}\!\bigl(z;f(x),\,H(x)^{-1}\bigr)+O(\sigma^{3}),
    \qquad
    H^{-1}=\sigma^{2}(J^{\top}J)^{-1}+O(\sigma^{4}).
\end{equation}  

For fixed $x$ the KL divergence between the optimal $q(z|x)$ and $p(z|x)$ is well approximated by
\begin{equation}
    \frac12\Bigl[
       \operatorname{tr}(H\Sigma)
       +(\mu-f(x))^{\top}H(\mu-f(x))
       -\log\det(\Sigma H)-d
    \Bigr],
\end{equation}
where $\mu=f_\mu(x)$ and the diagonal matrix $\Sigma=\operatorname{diag}f_{\sigma^2}(x)$.
The term in $\mu$ is minimized at $\mu=f(x)$; inserting this and differentiating with respect to each diagonal entry $s_k=(\Sigma)_{kk}$ gives the optimum
\begin{equation}\label{eq:s_k}
    s_k=(H^{-1})_{kk}.
\end{equation}
Substituting these back in to the KL divergence gives
\begin{equation}
    D_\text{KL}(q(z|x)\|p(z|x)) \approx \frac12\Bigl[
       \operatorname{tr}(R^{-1})
       +\log\det(R)-d
    \Bigr],
\end{equation}
where $R = \Sigma^{-1/2} H^{-1} \Sigma^{-1/2}$ is a correlation matrix. The divergence is minimized when $R=I$, meaning $H^{-1}$ and hence $H$ are diagonal. Therefore off-diagonal (but not diagonal) terms of $J^\top J$ are suppressed. 

Since the off-diagonal terms of $J^{\top}J$ are suppressed while the diagonal terms are not, $g^{*\prime}(z)^\top g^{*\prime}(z)$ of the optimal decoder will be approximately diagonal.
\end{proof}

\subsection{Proof of theorem 2}

\begin{theorem}[Local existence of orthogonal coordinates]
Let
\begin{equation}
  g : U \subset \mathcal{Z} \to \mathcal{X}
\end{equation}
be a smooth immersion on a connected open set $U$ (so $\operatorname{rank} g'(z) = d$ everywhere).
Set
\begin{equation}
  G = J_{g}^{\top} J_{g},
\end{equation}
the induced positive–definite metric on $U$.
Assume that every eigenvalue of $G$ is simple and varies smoothly, so one may choose a smooth orthonormal eigen-frame $\{e_{1},\dots,e_{d}\}$ with $G(e_{i},e_{i}) = \lambda_{i} > 0$.
For each $i$ define the rank-$1$ eigen-line field $\mathcal{L}_{i} = \operatorname{span}\{e_{i}\} \subset TU$.

The following statements are equivalent:
\begin{enumerate}
  \item \textbf{Orthogonal coordinates exist.}\;
        For every point $p \in U$ there is a neighborhood $V$ around $p$ and a diffeomorphism
        $\varphi : V \to W \subset \mathbb{R}^d$,
        $\;u = (u^{1},\dots,u^{d}) = \varphi(z)$, such that for
        $h := g \circ \varphi^{-1}$ the pull-back metric is diagonal:
        \begin{equation}
          h^{*}\langle\cdot,\cdot\rangle
          \;=\;
          \sum_{i=1}^{d} H_{i}^{2}(u)\,(du^{i})^{2},
          \qquad H_{i} > 0 \text{ smooth}.
        \end{equation}

  \item \textbf{Eigen-line fields are involutive.}\;
        For all $i \neq j$ one has
        \begin{equation}
          [e_{i},e_{j}] \;\in\; \operatorname{span}\{e_{i},e_{j}\}.
        \end{equation}
\end{enumerate}
\end{theorem}

In the proof, we use results from \cite{lee2013manifolds}, especially the Frobenius integrability theorem. Please refer to Chapter 19 of that work. The remaining chapters are a useful reference for differential geometry terminology and notation.

\begin{proof}
Let $U\subset\mathcal Z$ be the connected open set on which every
eigenvalue $\lambda_1,\dots,\lambda_d$ of the induced metric
\begin{equation}
    G = J_g^{\top} J_g
\end{equation}
is simple, and fix the smooth orthonormal eigen-frame
$\{e_1,\dots,e_d\}\subset TU$ satisfying
\begin{equation}
  G(e_i,e_j) \;=\; \lambda_i\,\delta_{ij},
  \qquad \lambda_i>0 .
\end{equation}
Denote the dual co‐frame by
$\{\omega^1,\dots,\omega^d\}\subset\Omega^1(U)$, so
$\omega^i(e_j)=\delta_{ij}$ for all $i,j$.

\textbf{Statement 1 implies statement 2}

Assume that there are \emph{orthogonal coordinates}
$u=(u^1,\dots,u^d)$ on some neighborhood $V\subset U$ such that
\begin{equation}\label{eq:diagMetric}
  h^{*}\langle\cdot,\cdot\rangle
  \;=\;
  \sum_{i=1}^{d} H_i^{2}(u)\,\bigl(du^{i}\bigr)^{2},
  \qquad
  h := g \circ \varphi^{-1},
  \quad
  H_i>0 .
\end{equation}
Set $\partial_i := \partial/\partial u^i$. From
\cref{eq:diagMetric} we obtain
\begin{equation}
  e_i \;=\; H_i^{-1}\,\partial_i ,
\end{equation}
and
\begin{equation}
  [\partial_i,\partial_j] \;=\; 0 .
\end{equation}
due to orthogonality. Hence
\begin{align}
  [e_i,e_j]
    &= H_i^{-1}H_j^{-1}[\partial_i,\partial_j]
       + (\partial_i H_j^{-1})\partial_j
       - (\partial_j H_i^{-1})\partial_i  \\
    &\in \operatorname{span}\{e_i,e_j\}, \nonumber
\end{align}
establishing involutivity of the eigen‐line fields.

\textbf{Statement 2 implies statement 1}

Suppose now that
\begin{equation}\label{eq:involutive}
  [e_i,e_j] \;\in\; \operatorname{span}\{e_i,e_j\}
  \quad\text{for all } i\neq j .
\end{equation}

We take the following steps:
\begin{enumerate}
    \item \textbf{Integrating the distributions.}
    The local Frobenius theorem \citep[Theorem 19.12]{lee2013manifolds} guarantees that,
    for every $i$, there exist a smooth function $u^i$ and a positive
    function $H_i$ such that
    \begin{equation}\label{eq:omega}
      \omega^i \;=\; \frac{H_i}{\sqrt{\lambda_i}}\,du^i.
    \end{equation}

    \item \textbf{Constructing the coordinate chart.}
    Because $\{\omega^1,\dots,\omega^d\}$ is linearly independent at each point, the differentials $\{du^1,\dots,du^d\}$ are independent by \cref{eq:omega}. Therefore the map
    \begin{equation}
        \varphi : V \to \mathbb R^{d},
        \qquad
        \varphi(z) \;=\; \bigl(u^1(z),\dots,u^d(z)\bigr),
    \end{equation}
    defined on a sufficiently small neighborhood $V\subset U$, has full rank and is a local diffeomorphism.

    \item \textbf{The metric in the new coordinates.}
    Using \cref{eq:omega},
    \begin{align}
      G
        &= \sum_{i=1}^{d} \lambda_i\,\omega^{i}\otimes\omega^{i}
           \;=\;
           \sum_{i=1}^{d}
           H_i^{2}\,(du^{i})^{2}.
    \end{align}
    Thus, writing $h := g\circ\varphi^{-1}$, we obtain
    \begin{equation}
        h^{*}\!\langle\,\cdot,\cdot\rangle
          \;=\;
          \sum_{i=1}^{d} H_i^{2}\,(du^{i})^{2},
    \end{equation}
    which is the diagonal form required.
\end{enumerate}

The two conditions are therefore equivalent.

\end{proof}

\subsection{Proof of theorem 3}

Recall the definition of the FIVE loss:
\begin{equation}\label{eq:five-loss}
    \mathcal{L}(x) = \mathbb{E}_{q(z|x)}[-\log p(x|z)] + \widetilde D_\text{KL}(q(z|x)\|p(z))
\end{equation}
where $\widetilde D_\text{KL}$ is the KL divergence with a substituted log-determinant term:
\begin{align}
D_{\text{KL}}(q(z|x) \| p(z)) &= \frac{1}{2} \left( \|f(x)\|^2 + \operatorname{tr}(\sigma^2 f'(x) f'(x)^\top) - d \right. \\
&\left. \quad - d \log \sigma^2 - \operatorname{tr}\left(\text{SG}[g'(f(x))] \cdot f'(x) \right) \right) \nonumber
\end{align}

In order to prove theorem 3, we first prove two lemmas which characterize the optimal solutions of the FIVE loss in the linear case.

\begin{lemma}[1-dimensional case]
\label{lem:five-optimum-1d}
Let $f: \mathbb{R} \to \mathbb{R}$ and $g: \mathbb{R} \to \mathbb{R}$ be linear maps, defined as $f(x) = v x$, $g(z) = w z$, with $v,w \in \mathbb{R}$. Let $q(x)$ be a zero-mean distribution with variance $\lambda > 0$. 

Then, the minimizers of the FIVE loss have the form
\begin{align}
w^{\star} &= \pm \sqrt{\lambda}, \\
v^{\star} &= \pm \frac{\sqrt{\lambda}}{\sigma^{2}+\lambda},
\end{align}
with the same sign in $w^{\star}$ and $v^{\star}$.
\end{lemma}

\begin{proof}
With $q(z|x) = \mathcal N(vx,\sigma^{2}v^{2})$ a direct expansion gives  
\begin{equation}
\mathcal L(w,v)
      =\frac{\lambda(1-wv)^{2}}{2\sigma^{2}}
       +\frac12w^{2}v^{2}
       +\frac{\lambda v^{2}}{2}
       +\frac12\sigma^{2}v^{2}
       -\mathrm{SG}[w]v
       +\text{const}.
\label{eq:five-1d-loss}
\end{equation}

The ordinary derivative w.r.t.\ $w$ is
\begin{equation}\label{eq:dL-dw}
\frac{\partial\mathcal L}{\partial w}
        = -\frac{\lambda v(1-wv)}{\sigma^{2}} + wv^{2}
        = 0.
\end{equation}
And likewise w.r.t.\ $v$ (obeying the stop-gradient operation) is
\begin{equation}\label{eq:dL-dv}
\frac{\partial\mathcal L}{\partial v}
      =-\frac{\lambda}{\sigma^{2}}\,w(1-wv)
       +w^{2}v
       +(\sigma^{2}+\lambda)v
       -w
       =0.
\end{equation}

Multiplying \cref{eq:dL-dw} by $\sigma^2/v$ (assuming $v\neq0$) and rearranging terms gives
\begin{equation}
    v = \frac{\lambda}{w(\sigma^2 + \lambda)}.
\end{equation}
Substituting this expression into \cref{eq:dL-dv}, multiplying by $w(\sigma^2+\lambda)$ and simplifying gives
\begin{equation}
    w^2 = \lambda.
\end{equation}
Both solutions $w \pm \sqrt{\lambda}$ have the same loss value, so the minimal solutions are as stated in the theorem.
\end{proof}

\begin{lemma}[Higher-dimensional case]
\label{lem:five-optimum-general-case}
Let $f: \mathbb{R}^n \to \mathbb{R}^d$ and $g: \mathbb{R}^d \to \mathbb{R}^n$, $d \le n$, be linear maps, defined as $f(x) = V^\top x$, $g(z) = W z$, with $V,W \in \mathbb{R}^{n \times d}$. Let $q(x)$ be a zero-mean distribution with covariance matrix $\Sigma_x \succ 0$. 

Diagonalize $\Sigma=U\operatorname{diag}(\lambda_1,\dots,\lambda_n)U^{\top}$  
with $U^{\top}U=I_n$ and $\lambda_i>0$. Let $S$ be the set of indices of the $d$ largest eigenvalues of $\Sigma$, with $U_S$ and $\Lambda_S$ the corresponding eigenvector and eigenvalue matrices. Let $R \in O(d)$ be orthogonal.

Then, the minimizers of the FIVE loss have the form
\begin{align}
W^{\star} &= U_S \Lambda_S^{1/2} R, \\
V^{\star} &= U_S \Lambda_S^{1/2} (\sigma^{2}I_d+\Lambda_S)^{-1} R.
\label{eq:lemma_opt_V}
\end{align}
\end{lemma}

\begin{proof}
Write $y=U^\top x$ so that the covariance matrix of $y$ is $\Lambda=\operatorname{diag}(\lambda_i)$. Define rotated parameters $\widetilde W = U^{\top}W$ and $\widetilde V = U^{\top}V$. We will solve for the optimal $(\widetilde W,\widetilde V)$ and then rotate the result back.

In the rotated frame the loss splits into a sum of $n$ independent one-dimensional objectives, each involving a row of $\widetilde W$ and $\widetilde V$ and the corresponding eigenvalue $\lambda_i$. Note that due to rotational symmetry in the latent space (i.e., replacing $(W,V)$ by $(WR,VR)$ for orthogonal $R$ leads to the same loss value), we can always choose $\widetilde W$ and $\widetilde V$ so that there is at most one nonzero value in each column. Call $\widetilde w_i$ the nonzero value of row $i$ of $\widetilde W$ (if it exists) and likewise for $\widetilde v_i$ and $\widetilde V$. The scalar problem was solved in \cref{lem:five-optimum-1d}, yielding the unique minimizer (up to sign)
\begin{equation}
(\widetilde w_i,\widetilde v_i)
=\bigl(\sqrt{\lambda_i},\;
        \tfrac{\sqrt{\lambda_i}}{\sigma^{2}+\lambda_i}
  \bigr).
\end{equation}

We have the maximum possible gain in the ELBO when selecting the largest $d$ eigenvalues and dropping the remaining $n-d$ eigenvalues. Therefore the optimal solution is when $\widetilde W$ has nonzero values only in the rows corresponding to these eigenvalues (and similarly for $\widetilde V$).

This leads to the stated solution, where we introduce orthogonal $R$ to account for the rotational symmetry of the latent space.
\end{proof}

\begin{theorem}
Let $f: \mathbb{R}^n \to \mathbb{R}^n$ and $g: \mathbb{R}^n \to \mathbb{R}^n$, be linear maps, defined as $f(x) = V^\top x$, $g(z) = W z$, with $V,W \in \mathbb{R}^{n \times n}$. Let $q(x)$ be a zero-mean distribution with covariance matrix $\Sigma_x \succ 0$. 

Then, in the limit $\sigma \to 0$, the global optimum of the FIVE loss corresponds to a model $p(x)$ that is a Gaussian distribution with mean zero and covariance $\Sigma_x$:
\begin{equation}
p(x) = \mathcal{N}(x; 0, \Sigma_x),
\end{equation}
which is the best possible solution attainable by any linear decoder $g$. Additionally, for nonzero $\sigma$, the pushforward of $p(z)$ through $g$ reproduces the same distribution:
\begin{equation}
g_\# p(z) = \mathcal{N}(x; 0, \Sigma_x).
\end{equation}
\end{theorem}

\begin{proof}
By application of \cref{lem:five-optimum-general-case} in the case $d=n$, we find that
\begin{equation}
    W^{\star} = U \Lambda^{1/2} R,
\end{equation}
and hence
\begin{equation}
    W^{\star} W^{\star\top} = U \Lambda U^\top = \Sigma_x.
\end{equation}
As a result, the pushforward of $p(z)$ through $g$ is equal to $\Sigma_x$, which is also equal to $p(x)$ in the zero $\sigma$ limit.
\end{proof}

\subsection{Derivation of free-form injective flow gradient}
We want to show that:
\begin{equation}
\nabla \frac{1}{2} \log \det H(x) \approx \nabla \operatorname{tr} \left( \text{SG}\left[f'(x)\right] \cdot g'(f(x)) \right).
\end{equation}

Recall the definitions $u(x, z) = -\log p(x, z)$ and $f(x) = \arg\min_z u(x,z)$. Also recall that $H(x)$ is the Hessian of $u(x, z)$ with respect to $z$, evaluated at $z = f(x)$, which we can compactly denote by $H(x) = u_{zz}(x,f(x))$.

Due to the definition of $u$:
\begin{equation}
    u(x,z) = \frac12 \|z\|^2 + \frac{1}{2\sigma^2} \|x-g(z)\|^2 + \text{const},
\end{equation}
we have
\begin{equation}
    u_z(x,z) = z^\top - \frac{1}{\sigma^2}(x-g(z))^\top g'(z)
\end{equation}
and
\begin{equation}
    u_{zx}(x,z) = -\frac{1}{\sigma^2}g'(z)^\top.
\end{equation}
Due to the definition of $f$, it must be that $u_z(x,f(x))=0$. By taking the total derivative with respect to $x$, we find
\begin{equation}
    \frac{d}{dx} u_z(x,f(x)) = u_{zx}(x,f(x)) + u_{zz}(x,f(x))f'(x) = 0,
\end{equation}
which immediately leads to
\begin{equation}
    H(x)f'(x) = \frac{1}{\sigma^2}g'(f(x))^\top.
\end{equation}
Multiplying by the pseudoinverse of $f'(x)$ and taking the inverse leads to
\begin{equation}
    H(x)^{-1} = \sigma^2 f'(x) g'(f(x))^{+\top}.
\end{equation}

With the assumption that second-order and higher derivatives of $g$ are negligible, $H$ simplifies to
\begin{equation}
    H(x) \approx I + \frac{1}{\sigma^2} g'(f(x))^\top g'(f(x)).
\end{equation}
Taking a derivative $\nabla$ w.r.t.\ the parameters of $g$, we have
\begin{equation}
    \nabla H(x) \approx \frac{2}{\sigma^2} g'(f(x))^\top \nabla g'(f(x)).
\end{equation}
Substituting this into Jacobi's formula for the gradient of the determinant leads to the desired expression
\begin{align}
    \nabla \frac12 \log\det H(x) &= \frac12\operatorname{tr}\left(H(x)^{-1}\nabla H(x)\right) \\ &\approx \operatorname{tr}\left(f'(x) g'(f(x))^{+\top} g'(f(x))^\top \nabla g'(f(x)) \right) \\ &= \operatorname{tr}\left(f'(x) \nabla g'(f(x)) \right) \\ &= \nabla \operatorname{tr} \left( \text{SG}\left[f'(x)\right] \cdot g'(f(x)) \right).
\end{align}

\section{Demonstration of constraints of diagonal covariance}

To explore the constraints imposed by different VAEs, we train FIVE, a full-covariance VAE (FC-VAE), and a standard (diagonal covariance) VAE on a toy dataset, using the same architecture for each. 
The dataset is generated by sampling $x, y \sim \mathcal{N}(0, 1)$, setting $z = x^2 + y^2$ and then adding isotropic Gaussian noise in $\mathbb{R}^3$ with standard deviation $\sigma = 0.2$. It is specially constructed so that the posterior of the ground truth generative process is not diagonal. 
The architecture consists of a fully connected net with a width of $256$, a depth of $2$ and SiLU activation functions. We trained the model using the Adam optimizer with a learning rate of $10^{-4}$ and weight decay of $0.001$. We generated $10000$ samples, training each model over $100$ epochs using a batch size of $50$. We further fixed the noise value across models to equal the true noise value of the toy dataset.
For further regularization we used orthogonal initialization as proposed by \cite{saxe2014exact}.

The different learned parameterizations through the decoder are displayed in \cref{fig:learned parameterization} and the corresponding encodings are visualized in \cref{fig:learned encodings}. In both figures, we show the ground truth (``explicit'') parameterization or encoding compared to those learned by a FIVE, a FC-VAE, and a standard VAE.
To visualize the learned parametrization of the manifold we mapped a latent-space coordinate grid through the decoder and projected to the $xy$ plane. Correspondingly, to visualize the encoding we mapped the ground truth parameterization of the manifold through the encoder.

We see that the VAE learns a parametrization similar to polar coordinates: $(x(u, v), y(u, v), z(u, v)) = (u\cos v, u \sin v, u^2)$, leading to an approximately diagonal pullback metric. In contrast, the FIVE learns a parametrization more similar to the explicit parametrization $(x(u, v), y(u, v), z(u, v)) = (u, v, u^2 + v^2)$.
The FC-VAE seems to learn a parametrization more similar to the VAE than to the FIVE. This might be explained by the approximation gap being more relevant than the capacity of the variational model as argued by \cite{cremer2018inference} and as the FIVE learns the covariance matrix implicitly its approximation gap during training might be smaller.

We further provide a comparison between the true posteriors and the learned approximation of each model evaluated at the point $(x, y, z) = (1, 1, 2)$ in \cref{fig:learned posteriors}. Here we can see that the VAE not only learns a diagonal posterior $q(z|x)$ but also shapes the decoder such that the true posterior $p(z|x)$ is also approximately diagonal, in agreement with \cref{thm:diagonal-pullback}.

\begin{figure}[t]
  \centering
  \begin{subfigure}[b]{0.24\textwidth}
    \centering
    \includegraphics[trim={1cm 0 0 0},clip, height=\textwidth]{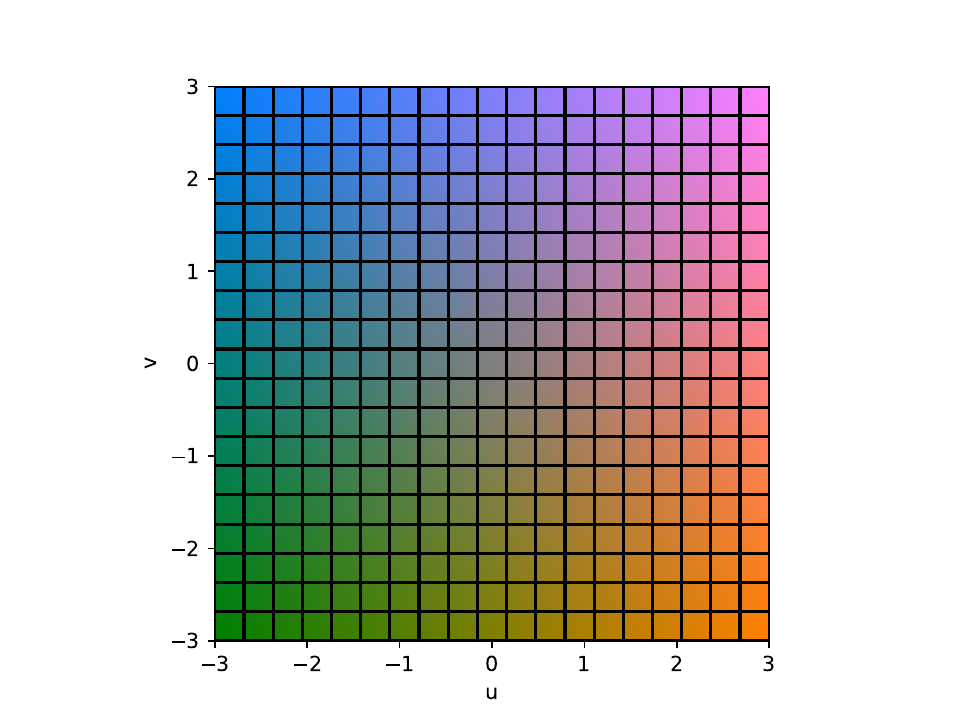}
    \caption{Explicit}
  \end{subfigure}%
  \hfill
  \begin{subfigure}[b]{0.24\textwidth}
    \centering
    \includegraphics[trim={1cm 0 0 0},clip,height=\textwidth]{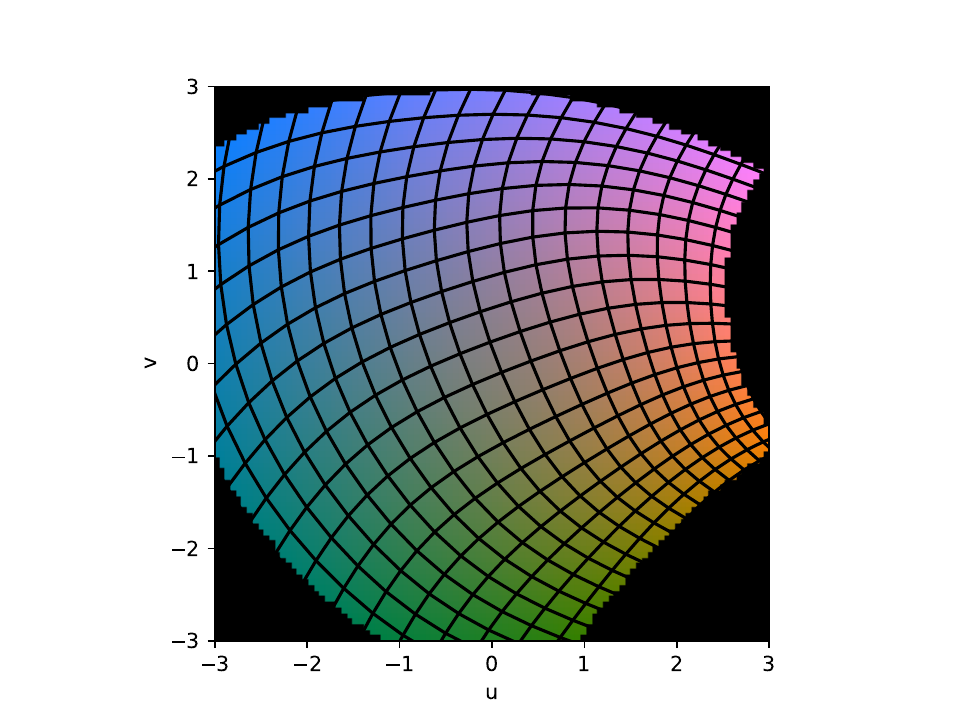}
    \caption{FIVE}
  \end{subfigure}%
  \hfill
  \begin{subfigure}[b]{0.24\textwidth}
    \centering
    \includegraphics[trim={1cm 0 0 0},clip,height=\textwidth]{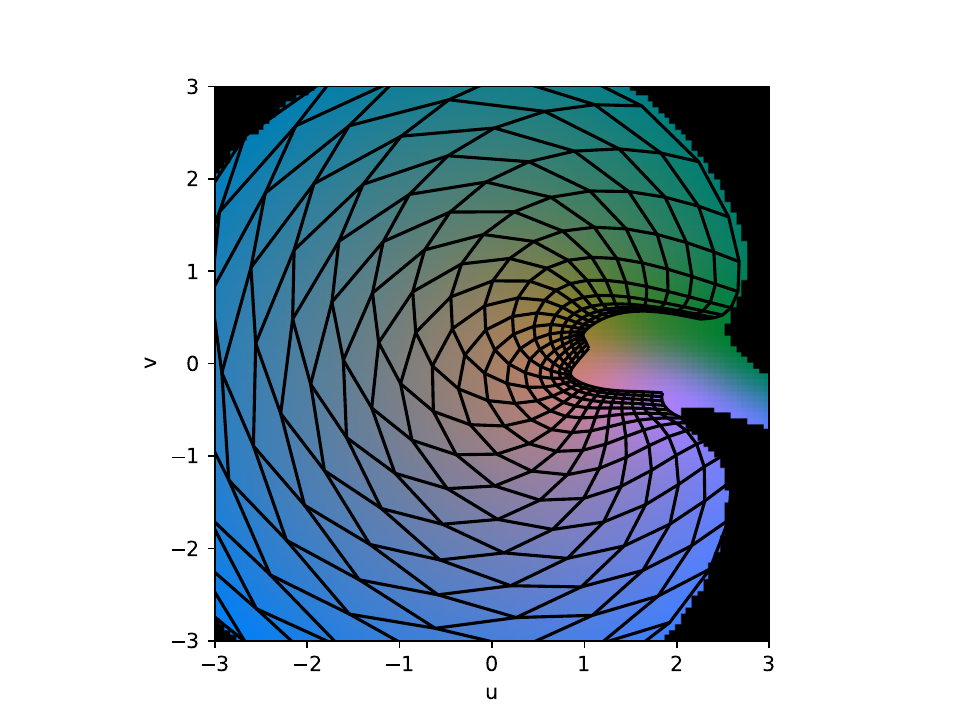}
    \caption{FC-VAE}
  \end{subfigure}
  \hfill
  \begin{subfigure}[b]{0.24\textwidth}
    \centering
    \includegraphics[trim={1cm 0 0 0},clip,height=\textwidth]{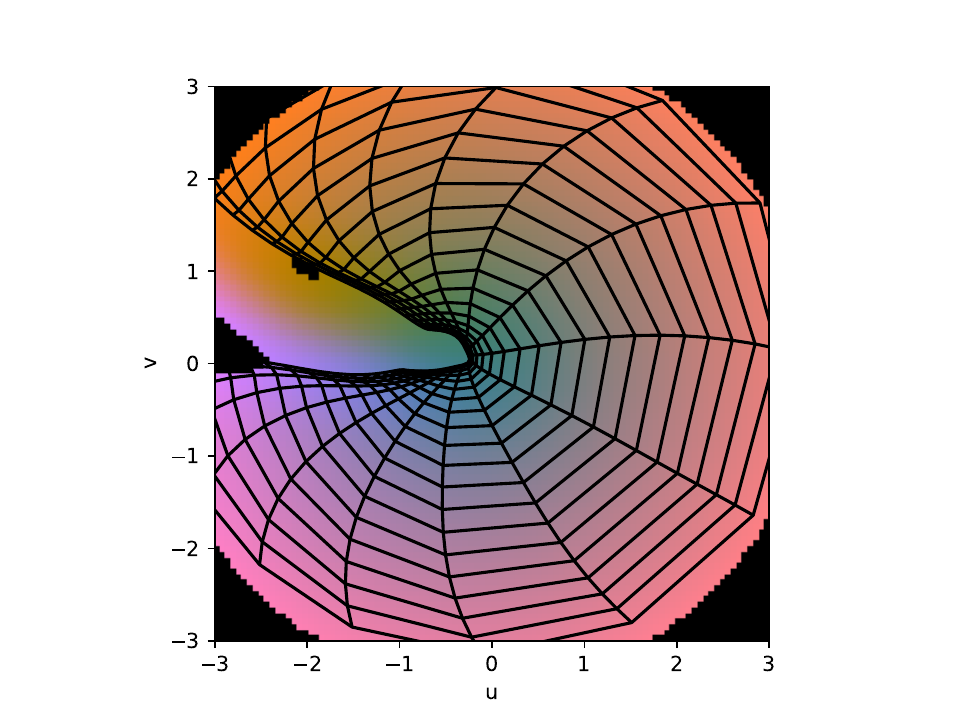}
    \caption{VAE}
  \end{subfigure}
  
  \caption{Parameterizations of the manifold. Left: ground truth explicit parameterization $z = x^2+y^2$, projected to the $xy$ plane. The remaining three sub-figures show parameterizations learned by different models, obtained by mapping a latent-space coordinate grid through the decoder and projecting to the $xy$ plane.} 
  \label{fig:learned parameterization}
\end{figure}

\begin{figure}[ht]
  \centering
  \begin{subfigure}[b]{0.24\textwidth}
    \centering
    \includegraphics[trim={1cm 0 0 0},clip,height=\textwidth]{figures/vae_parametrization_experiments/encoder_representation_explicit.pdf}
    \caption{Explicit}
  \end{subfigure}
  \hfill
  \begin{subfigure}[b]{0.24\textwidth}
    \centering
    \includegraphics[trim={1cm 0 0 0},clip,height=\textwidth]{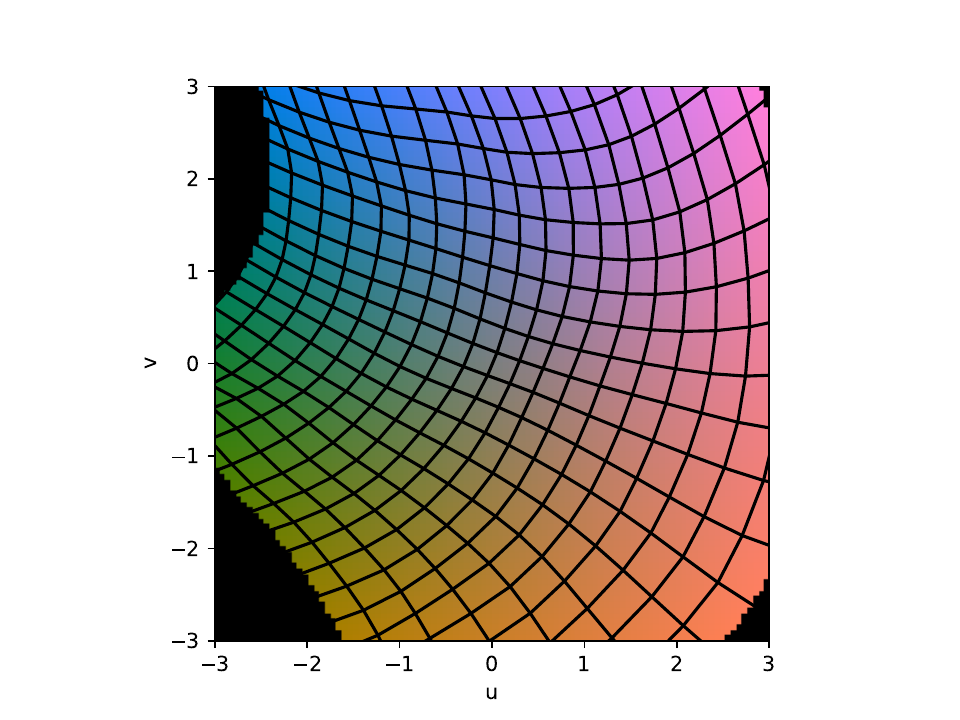}
    \caption{FIVE}
  \end{subfigure}%
  \hfill
  \begin{subfigure}[b]{0.24\textwidth}
    \centering
    \includegraphics[trim={1cm 0 0 0},clip,height=\textwidth]{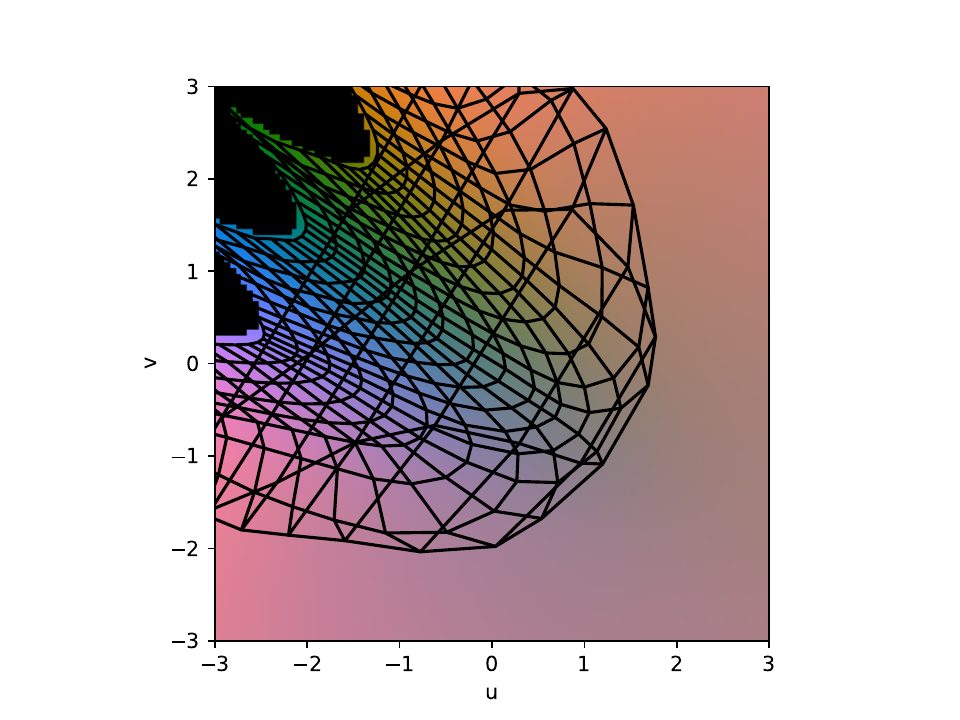}
    \caption{FC-VAE}
  \end{subfigure}
  \hfill
  \begin{subfigure}[b]{0.24\textwidth}
    \centering
    \includegraphics[trim={1cm 0 0 0},clip,height=\textwidth]{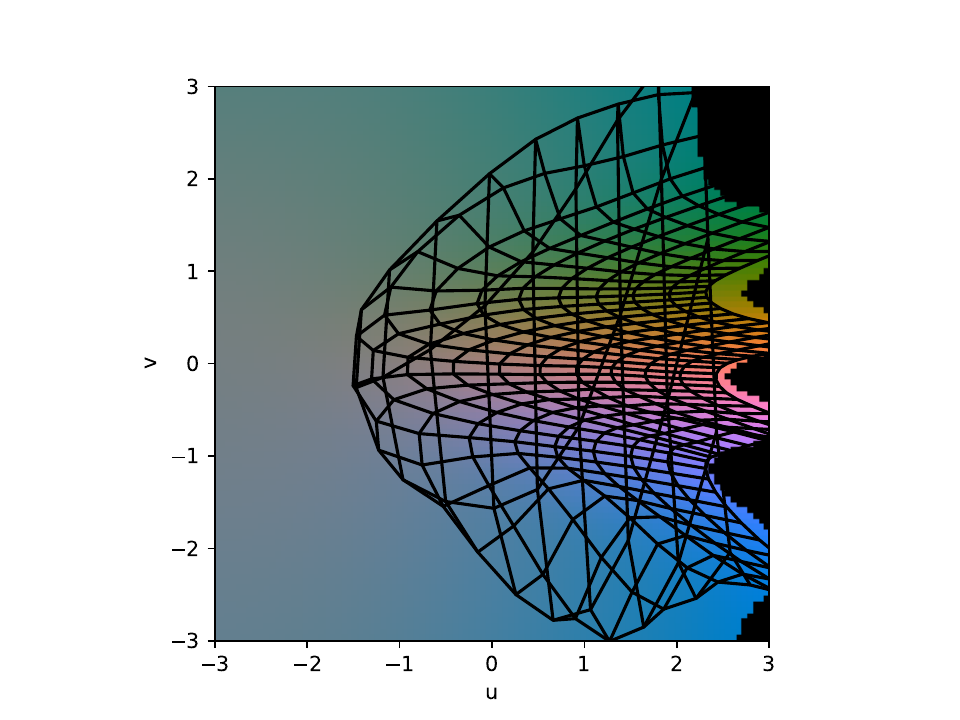}
    \caption{VAE}
  \end{subfigure}
  
  \caption{Encodings of the manifold. Left: ground truth latent-space coordinate grid. The remaining 3 sub-figures show encodings learned by different models, obtained by mapping the ground truth parameterization of the manifold through the encoder.}
  \label{fig:learned encodings}
\end{figure}

\begin{figure}[ht]
  \centering
  \begin{subfigure}[b]{0.33\textwidth}
    \centering
    \includegraphics[trim={2cm 0 0 0},clip, height=\textwidth]{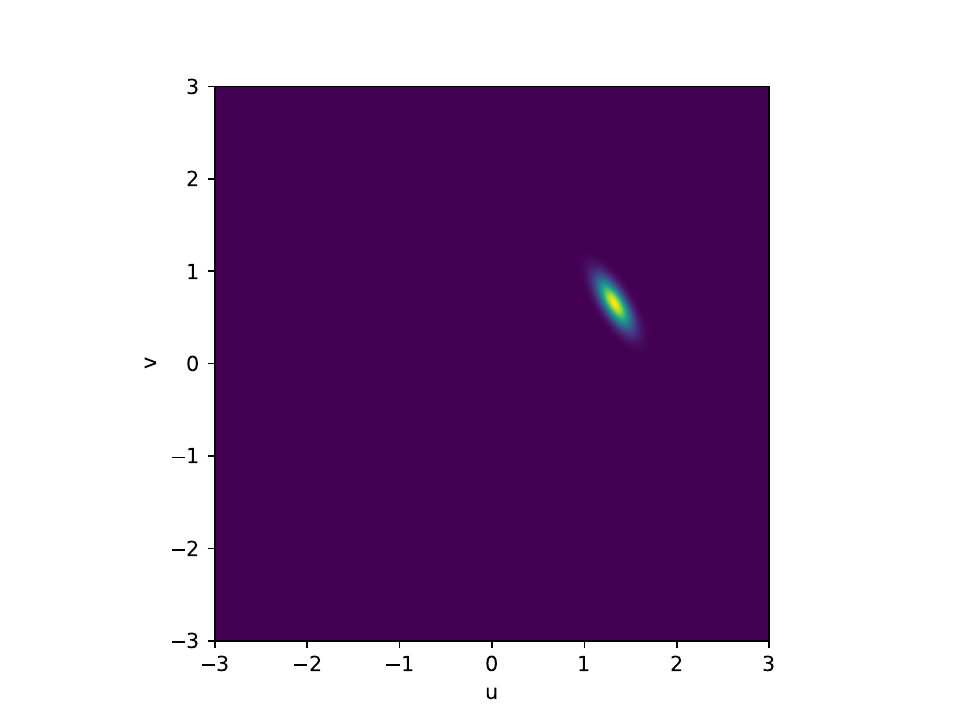}
    \caption{FIVE (learned)}
  \end{subfigure}%
  \hfill
  \begin{subfigure}[b]{0.33\textwidth}
    \centering
    \includegraphics[trim={2cm 0 0 0},clip,height=\textwidth]{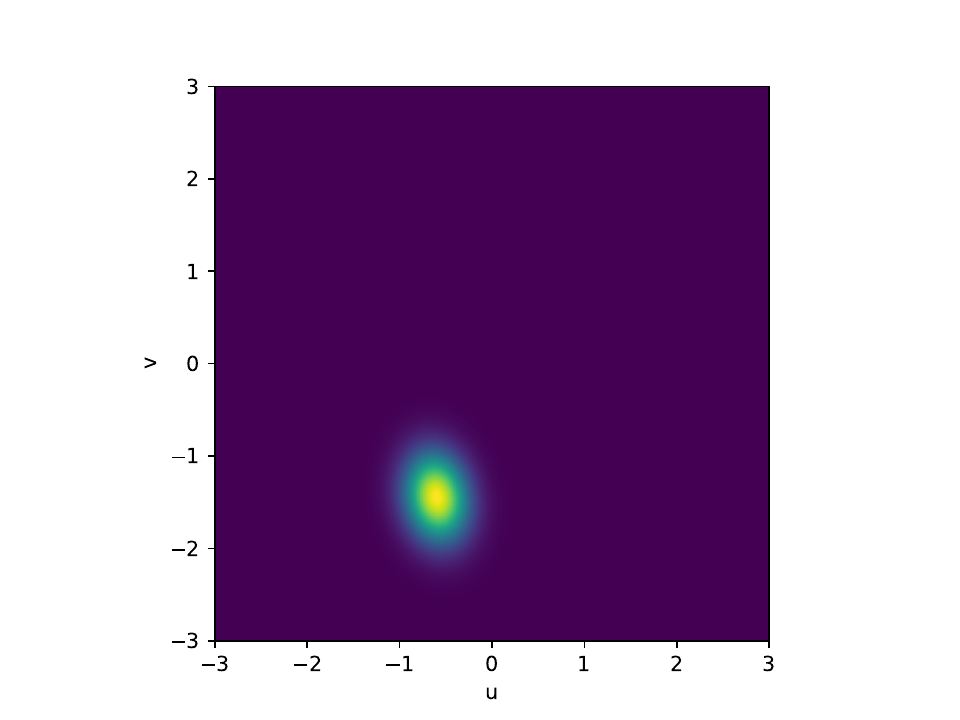}
    \caption{FC-VAE (learned)}
  \end{subfigure}%
  \hfill
  \begin{subfigure}[b]{0.33\textwidth}
    \centering
    \includegraphics[trim={2cm 0 0 0},clip,height=\textwidth]{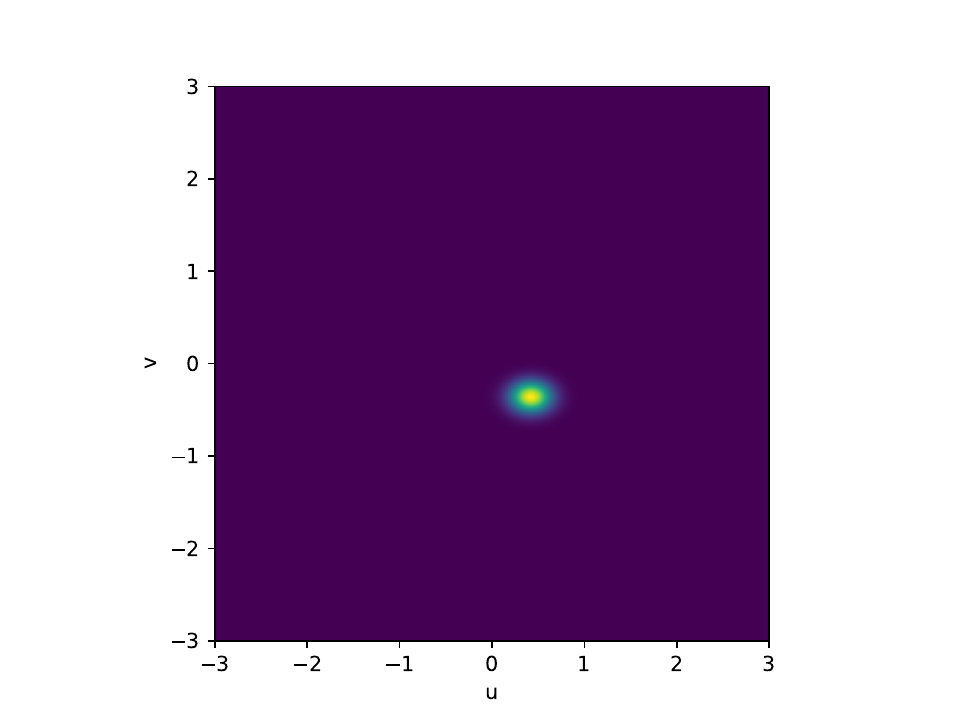}
    \caption{VAE (learned)}
  \end{subfigure}
  \begin{subfigure}[b]{0.33\textwidth}
    \centering
    \includegraphics[trim={2cm 0 0 0},clip, height=\textwidth]{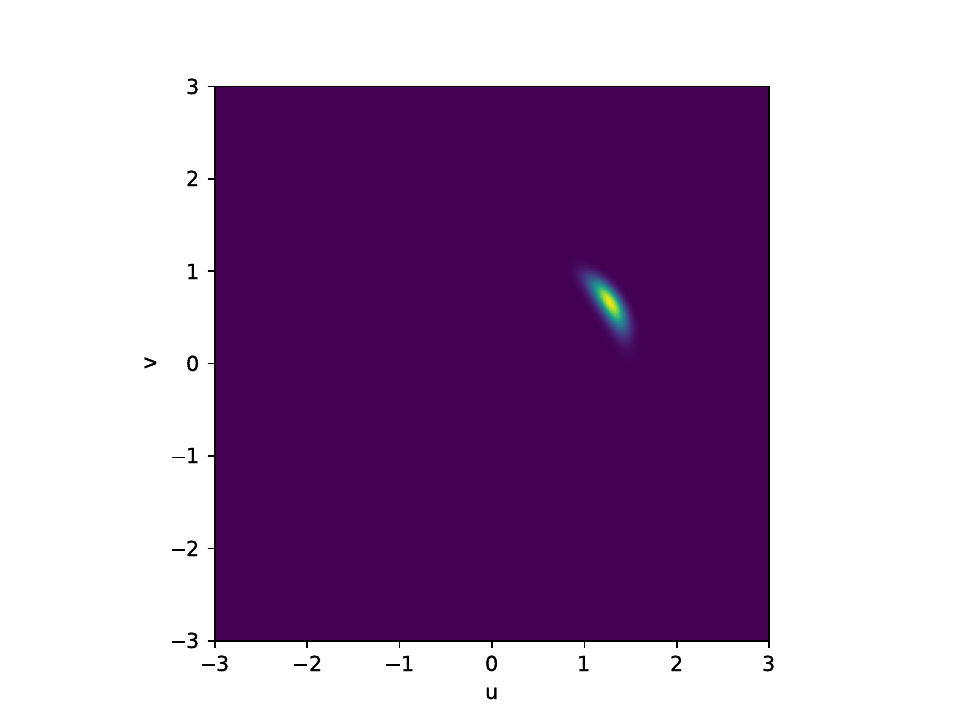}
    \caption{FIVE (true)}
  \end{subfigure}%
  \hfill
  \begin{subfigure}[b]{0.33\textwidth}
    \centering
    \includegraphics[trim={2cm 0 0 0},clip,height=\textwidth]{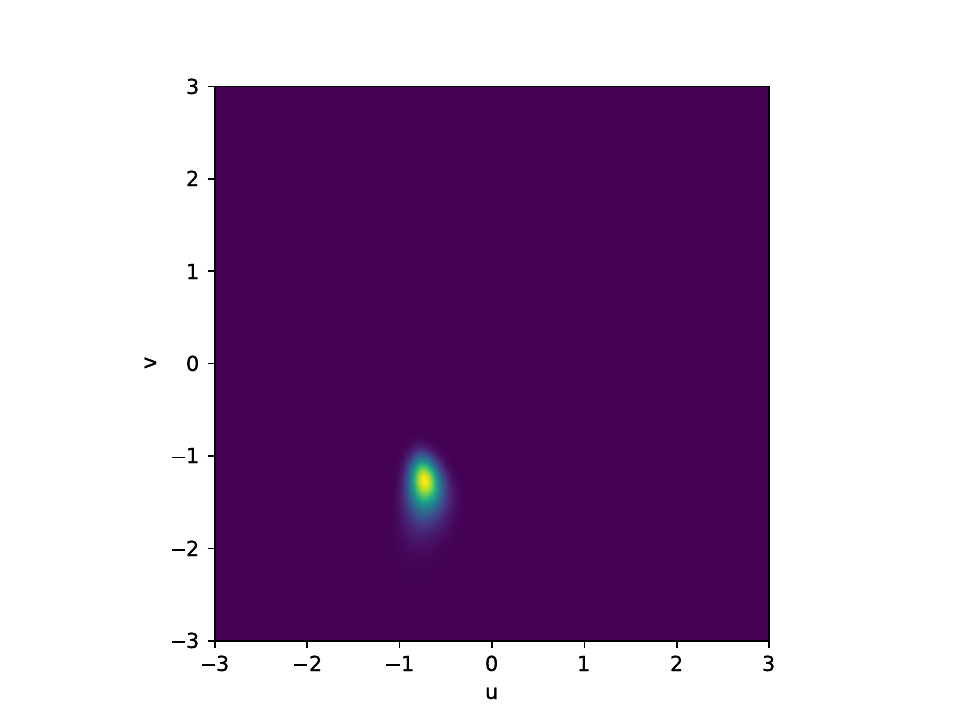}
    \caption{FC-VAE (true)}
  \end{subfigure}%
  \hfill
  \begin{subfigure}[b]{0.33\textwidth}
    \centering
    \includegraphics[trim={2cm 0 0 0},clip,height=\textwidth]{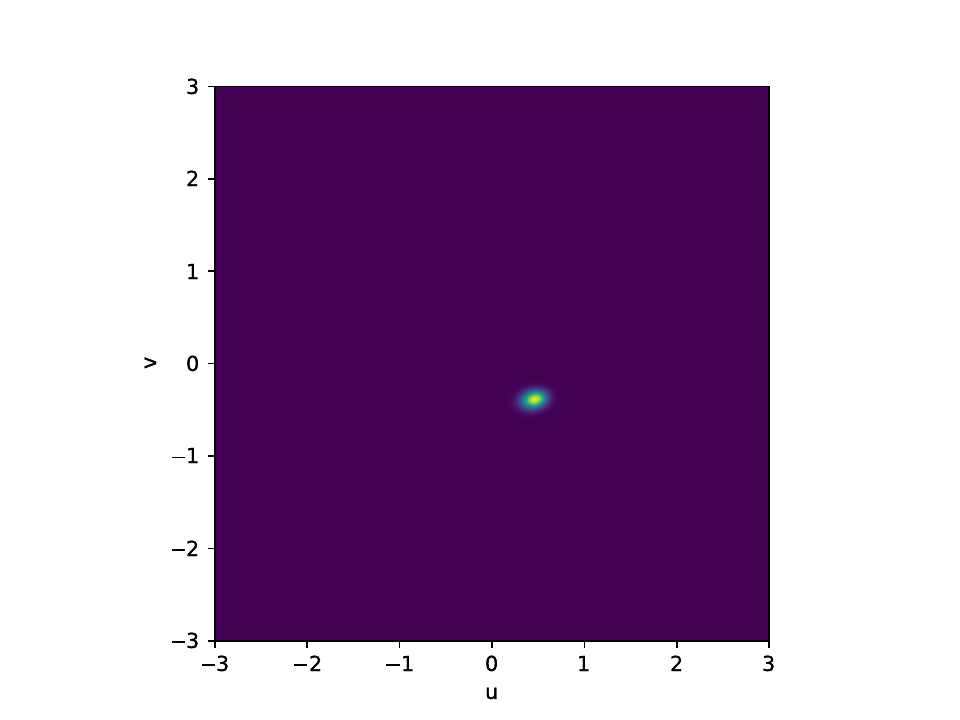}
    \caption{VAE (true)}
  \end{subfigure}
  
  \caption{Different learned variational posteriors $q(z|x)$ (top row) compared to the true posterior $p(z|x)$ defined by the decoder (bottom row), for three different models (one model per column). Posteriors are evaluated at the data-space point $(1, 1, 2)$.}
  \label{fig:learned posteriors}
\end{figure}

\section{Experimental details}

\subsection{Libraries and datasets}

We used the Python libraries NumPy \citep{numpy}, Matplotlib \citep{matplotlib}, PyTorch \citep{paszke2019pytorch}, and Guild AI \citep{guild}. We also employed the MNIST dataset \citep{lecun1998gradient} and the CIFAR-10 dataset \citep{krizhevsky2009learning} in our experiments.

\subsection{Computational resources}

Experiments were conducted on an Nvidia RTX 3090 graphics card with 24GB of VRAM. The total computational budget, including prototyping and hyperparameter exploration, was around 100 GPU-hours.

\end{document}